\documentclass[preprint,12pt]{elsarticle}

\usepackage{amssymb, amsfonts, amsmath, amsthm, subcaption, array, color, multirow, float, graphicx, soul, booktabs, fancyhdr}

\usepackage[breaklinks=true, colorlinks, linkcolor = blue, urlcolor  = blue, citecolor = blue, anchorcolor = blue, bookmarks=true]{hyperref}

\newcommand{\R}{\mathbb{R}}

\newcommand{\dsp}{\displaystyle}

\newtheorem{prop}{Proposition}

\usepackage{lineno}

\journal{Computer Speech and Language}

\date{}
\makeatletter
\def\ps@pprintTitle{%
  \let\@oddhead\@empty
  \let\@evenhead\@empty
  \let\@oddfoot\@empty
  \let\@evenfoot\@oddfoot
}
\makeatother

\begin{document}

\begin{frontmatter}



\title{Breaking Time Invariance: Assorted-Time Normalization for RNNs}


\author{Cole Pospisil\corref{cor1}}
\ead{cole.pospisil@uky.edu}
\author{Vasily Zadorozhnyy\corref{cor1}}
\ead{vasily.zadorozhnyy@uky.edu}
\author{Qiang Ye\corref{cor2}}
\ead{qye3@uky.edu}
\cortext[cor1]{These authors contributed equally}
\cortext[cor2]{Research supported in part by NSF under grants  DMS-2208314 and DMS-1821144.}
 \address[]{University of Kentucky, 719 Patterson Office Tower, Lexington, 40506, KY, USA}

\begin{abstract}
Methods such as Layer Normalization (LN) and Batch Normalization (BN) have proven to be effective in improving the training of Recurrent Neural Networks (RNNs). However, existing methods normalize using only the instantaneous information at one particular time step, and the result of the normalization is a preactivation state with a time-independent distribution. This implementation fails to account for certain temporal differences inherent in the inputs and the architecture of RNNs. Since these networks share weights across time steps, it may also be desirable to account for the connections between time steps in the normalization scheme. In this paper, we propose a normalization method called Assorted-Time Normalization (ATN), which preserves information from multiple consecutive time steps and normalizes using them. This setup allows us to introduce longer time dependencies into the traditional normalization methods without introducing any new trainable parameters. We present theoretical derivations for the gradient propagation and prove the weight scaling invariance property. Our experiments applying ATN to LN demonstrate consistent improvement on various tasks, such as Adding, Copying, and Denoise Problems and Language Modeling Problems.
\end{abstract}




\end{frontmatter}


\section{Introduction}
The Recurrent Neural Network (RNN)~\cite{10.5555/104279.104293, doi:10.1073/pnas.79.8.2554}, and variants such as Long Short Term Memory (LSTM)~\cite{Hochreiter:1997:LSM:1246443.1246450} or Gated Recurrent Unit (GRU)~\cite{https://doi.org/10.48550/arxiv.1406.1078, 8668730,https://doi.org/10.48550/arxiv.2208.06496}, are some of the core architectures used for modeling time-series data in Deep Learning today. While LSTMs and GRUs are effective in avoiding problems with vanishing gradients, all of these recurrent models are still subject to issues with exploding gradients, as well as over-fitting. One of the most successful ideas that have been introduced over the years is the normalization of RNNs using methods such as Layer Normalization (LN)~\cite{ba2016layer} and Batch Normalization (BN)~\cite{pmlr-v37-ioffe15,CooijmansBLGC17}. These methods recenter and rescale the preactivation information using the statistics of that time step. This allows for the norm of the model's states and gradients to be controlled, which speeds up training and prevents exploding gradients. 

While these normalization methods have been successful, their applications to RNNs do not involve adaptations to some of the primary characteristics of this class of models, namely that the variation across time imparts usable information. For example, the LN or BN models are invariant to the scaling in the input at any time step and are therefore independent of the norm of the input vector at each time step. Depending on applications, this may have devastating consequences. Additionally, LN and BN produce a preactivation state with a distribution that is invariant across time. Such time invariance properties may impede the architectural structure of RNNs ability to fully exploit the temporal dependencies. Since RNNs share weights across time steps, it would be quite natural to introduce this dependency into the normalization method as well. An attempted version of this involving averaging statistics across time was mentioned in ~\cite{CooijmansBLGC17} but was unsuccessful and was presented without much detail. It appears that simply averaging over every time step is an overcorrection that makes the statistics susceptible to diluted averages and loses effectiveness further into the sequence. Instead, we argue that, by collecting the mean and variance across a smaller subsequence, one is able to gain the benefits of these time dependencies without overly weakening the impact of a single time step.

In this paper, we propose a normalization method called  Assorted-Time Normalization (ATN), which preserves information from multiple consecutive time steps and normalizes using them. Our ATN method can be combined with other normalization methods such as LN and BN that normalize input information along some dimensions but not time. It maintains a short-term memory of the previous $k$ time steps, which allows it to account for the temporal dependencies in a way in which previous methods were incapable. We use that memory to calculate the statistics with respect to which we normalize, giving us an output that has a controlled mean and variance while still being capable of changing between time steps. By using just a limited subsequence at each point in time, we are able to avoid the problems that come from using all or none of the sequences and find the length best suited to the dataset. Since this process just adds a time component to the normalization method, it is adapting without the introduction of any new learnable parameters.

We present theoretical derivations for the gradient propagation and prove the weight scaling invariance property. Our experiments demonstrate consistent improvement using our method on a variety of tasks, such as Adding, Copying, and Denoise  Problems as well as Language Modeling Problems. Our code is available at \href{https://github.com/vasily789/atn}{https://github.com/vasily789/atn}.

\section{Related Work}\label{section:related_work}

One of the earliest attempts to use some sort of normalization technique throughout model layers was Batch Normalization (BN)~\cite{pmlr-v37-ioffe15}. It was proposed for Fully Connected (FC) and Convolutional (CNN) Neural Networks for normalization of network activations across the batch dimension. BN is known often to provide a more stable and accelerated training regimen while improving generalizations. The Instance Normalization (IN)~\cite{ulyanov2017instance} method, contrary to BN, acts like contrast normalization and has primarily been used for image-containing datasets. The paper points out that the output stylized images should not rely on the contrast of the input image content, and hence normalizing the instances helps. The Group Normalization (GN)~\cite{Wu_2018_ECCV} method, which is primarily used for CNNs, normalizes a 3D feature in a convolutional layer by dividing its channels into groups and then normalizing the features in the group in all three dimensions.

Consider the typical structure of an RNN, also known as an RNN cell:

\begin{align}
    h^{(t)}&=f\left(W_{h}h^{(t-1)}+W_{x}x^{(t)}+\beta_h\right)\label{rnn_equation:1}\\
    y^{(t)}&=W_{y}h^{(t)}+\beta_y\label{rnn_equation:2}
\end{align}
where $f$ is a nonlinear activation function. 

The Recurrent Batch Normalization~\cite{CooijmansBLGC17} method applies BN to the hidden-to-hidden and memory cell parts of the LSTM model, which aims to reduce the internal covariate shift between consecutive time steps. ~\cite{NIPS2016_ed265bc9} proposed a Weight Normalization (WN) method. Their idea lies in decoupling the magnitude from the direction of the weight vector to change the parameters of the network, which helps with speeding up learning. Unfortunately, WN appears not widely used in practice due to its limited stability compared to BN~\cite{gitman2017comparison}. 

Layer Normalization (LN) was proposed in~\cite{ba2016layer} to normalize activations along the hidden dimension for both FC networks and RNNs and has since become very popular in RNNs. LN normalizes the preactivation state as follows:

\begin{align}
    h^{(t)}&=f\left(LN\left(W_{h}h^{(t-1)}\right)+LN\left(W_{x}x^{(t)}\right)+\beta_h\right)\label{lnrnn_equation:1}\\
    y^{(t)}&=LN\left(W_{y}h^{(t)}\right) + \beta_y\label{lnrnn_equation:2}
\end{align}
where the LN operator is defined by

\begin{align}
    \mu_{t} = \dfrac{1}{n}\sum_{i=1}^n a_i^{(t)}\quad \sigma_{t}^2 = \dfrac{1}{n}\sum_{i=1}^n\left(a_i^{(t)}-\mu_{t}\right)^2 \label{ln_equations:1}\\
    y^{(t)} = LN(a^{(t)};\gamma,\beta)=\gamma\odot \dfrac{a^{(t)}-\mu_{t}}{\sqrt{\sigma_t^2 + \varepsilon}} + \beta \label{ln_equations:2}
\end{align}

Such a setup helps to get rid of the BN batch dependency and simplifies the application to RNNs.

More recently, Adaptive Normalization (AdaNorm)~\cite{NEURIPS2019_2f4fe03d} made a thorough analysis of LN, and concluded that the rescaling and recentering factors, $\gamma$ and $\beta$ in (\ref{ln_equations:2}), are not as essential as the backward gradients of the mean and variance inside of the LN method. In addition, they proposed a new method, AdaNorm, which replaces weight and bias with some new transformation function.

\section{Assorted Time Normalization}\label{section:theory}

One undesirable property of the adaptation of LN to RNNs is that the statistics for the normalization are calculated at each time step, resulting in a post-normalization state which has mean and variance that are invariant across time. This prevents the model from effectively representing the shifting distributions across time that might be critical in modeling sequential data. For example, the normalization $LN\left(W_{x}x^{(t)}\right)$ in (\ref{lnrnn_equation:1}) is invariant to scaling in $x^{(t)}$, which restricts the model from learning the changing norm of $x^{(t)}$. This may be mitigated by including a bias in the linear term, which is often used in implementations; see section \ref{abl_mnist} for more discussion. Most of the above discussions also apply to BN. 

We propose a new normalization method to break this time invariance. Consider a sequence $\mathbf{a} = \left\{a^{(t)}\right\}\subset {\mathbb R}^n$ produced in an RNN, such as the preactivation state that we wish to normalize. At time step $t$ of the RNN, we maintain a memory of the previous $k$ entries, $\mathbf{a}_k^{(t)} = \left\{a^{(t-k+1)}, \ldots, a^{(t-1)}, a^{(t)}\right\} \subset \mathbf{a}$, in the normalization layer, using this extended set to compute the mean and variance to be used for normalization. This can be combined with other normalization methods. Combining with Layer Normalization, for example, these statistics are calculated at time-step $t$ as follows: 

\begin{align}
    \mu_{t,k}&=\dfrac{1}{nk}\sum_{j=0}^{k-1}\sum_{s=1}^{n}a_s^{(t-j)}\label{atn_stats:1}\\ 
    \sigma_{t,k}^2&=\dfrac{1}{nk}\sum_{j=0}^{k-1}\sum_{s=1}^{n}\left(a_s^{(t-j)}-\mu_{t,k}\right)^2\label{atn_stats:2} 
\end{align}

\begin{figure*}[t]
    \centering
    \includegraphics[width=\textwidth]{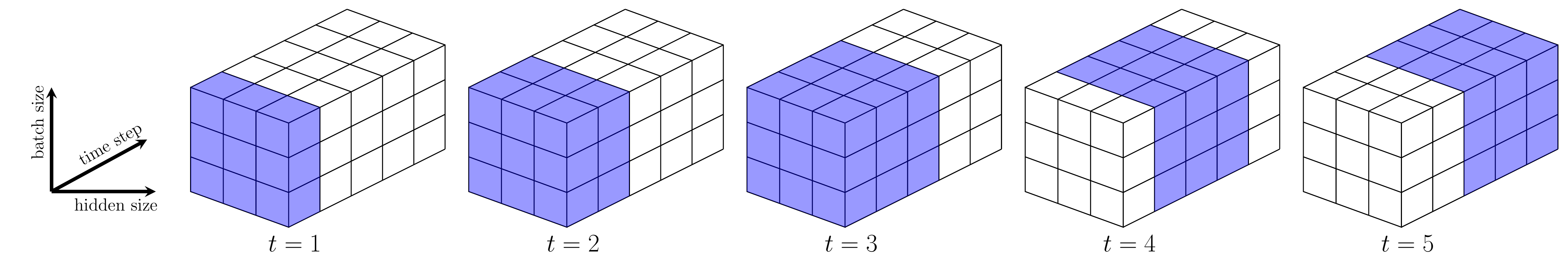}
    \caption{Illustration of the ATN method combined with LN using $k=3$ time steps: Consider the preactivation state tensor in $\R^{5\times 3 \times 3}$. At $t=1$, we use a standard LN; at $t=2$, we normalize using information from time steps $1,2$; at $t=3$ ATN method uses information from time step $t=1,2,3$; at $t=4$, we normalize with respect to time steps $t=2,3,4$; and so on after that.} 
    \label{fig:kLN_representation}
\end{figure*}

See Figure \ref{fig:kLN_representation} for a visual depiction of our method.

Once these statistics are calculated, we then normalize only the current term $a^{(t)}$ and optionally recenter and rescale using $\gamma$ and $\beta$, two trainable parameters shared across time while adding a small epsilon to the variance to prevent division by zero, similar to the LN method in (\ref{ln_equations:2}). 

\begin{align}
    y^{(t)}=ATN(\mathbf{a}_k^{(t)};\gamma,\beta) &:= \gamma\odot\dfrac{a^{(t)}-\mu_{t,k}}{\sqrt{\sigma_{t,k}^2 + \varepsilon}} + \beta
\end{align}

This differs from the process in (\ref{ln_equations:1}) and (\ref{ln_equations:2}) in that we include multiple time steps in our statistic calculations, giving us a double sum instead of the single one for LN. This definition of the statistics is more stable in time, at least for large $k$, changing modestly at each time step with only one term in the set being replaced. This results in a normalized output that is not expected to have a uniform mean and variance across time steps. We argue that this is desirable for sequential problems. Having this potential for variation allows for the model to account for changing norms of the inputs across the sequence, providing additional information about the distribution that is lost with previous methods. 

We may consider using all previous terms in the sequence to compute the statistics, but they will have more variation in early time steps than in later ones. By keeping only $k$ time steps and not the entire sequence, the statistics will vary gradually across time, and we are able to fix the memory and computational costs, which could be significant for long sequences. 
 
Using the information from multiple time steps also effectively provides a larger set on which to calculate statistics. This allows a more accurate approximation to gain a clearer glimpse at the underlying distribution of the dataset. In other words, ATN uses statistics over a larger set that is more stable across time so that the normalized state can retain more variations in time. In contrast, the traditional normalization methods use high-frequency statistics at each time step to produce a normalized state that becomes time-invariant. In particular, the ATN network depends on the scaling of the input vector at a time step while LN and BN do not. However, ATN preserves the desirable weight scaling invariant property, which we show as follows:

Let $H$ and $\Tilde{H}$ be weight matrices for two sets of model parameters, $\theta$ and $\Tilde{\theta}$ respectively, which differ by a scaling factor of $\delta$, i.e. $\Tilde{H}=\delta H$. Then the outputs of ATN are the same:

\begin{align}\label{der:rescaling_atn}
    \Tilde{y}^{(t)}&=\dfrac{\gamma}{\Tilde{\sigma}_{t,k}}\odot\left(\Tilde{H}a^{(t)}-\Tilde{\mu}_{t,k}\right)+\beta\ =\dfrac{\gamma}{\sigma_{t,k}}\odot\left(Ha^{(t)}-\mu_{t,k}\right)+\beta=y^{(t)}
\end{align}
where $\Tilde{\sigma}_{t,k}=\delta {\sigma}_{t,k}$ and $\Tilde{\mu}_{t,k}=\delta {\mu}_{t,k}$. This invariance property makes the ATN network independent of the norm of $H$, mitigating the exploding/vanishing gradient problems.

It is also easy to see that ATN is also invariant to the rescaling of the whole input sequence but not invariant under the rescaling of an individual element in the sequence. 

During training, we backpropagate the gradients with respect to the model parameters. With ATN, a key step is to propagate the gradient through the normalization layer, i.e. $\dfrac{\partial y_i^{(t)}}{\partial a_i^{(t-m)}}$. The following proposition gives the formulas for computing these derivatives.  The proof is provided in \ref{Supp_Mater:A}.

\begin{prop}\label{prop:gradients}
Consider ATN for a sequence $\mathbf{a} = \{a^{(t)}\}\subset {\mathbb R}^n$ produced in a RNN and let $y^{(t)}=ATN(\mathbf{a}_k^{(t)};\gamma,\beta)$. Then, for $0 \leq m \leq k-1$, we have:

\begin{equation}
    \resizebox{0.9\hsize}{!}{$
    \dfrac{\partial y_i^{(t)}}{\partial a_i^{(t-m)}} = \gamma\odot\dfrac{\displaystyle\dfrac{\partial a_i^{(t)}}{\partial y_i^{(t-m)}}\dfrac{\partial y_i^{(t-m)}}{\partial a_i^{(t-m)}} - \dfrac{\partial \mu_{t,k}}{\partial a_i^{(t-m)}}}{\sqrt{\sigma_{t,k}^2 + \varepsilon}} - \gamma\odot\dfrac{a_i^{(t)}-\mu_{t,k}}{2\left(\sigma_{t,k}^2+\varepsilon\right)^{3/2}}\dfrac{\partial \sigma_{t,k}^2}{\partial a_i^{(t-m)}}
    $}
\end{equation}
where

\begin{equation}
    \dfrac{\partial \mu_{t,k}}{\partial a_i^{(t-m)}} = \dfrac{1}{nk}\sum_{j=0}^{m}\dfrac{\partial a_i^{(t-j)}}{\partial a_i^{(t-m)}}
\end{equation}
\begin{equation}
    \resizebox{0.9\hsize}{!}{$
    \dfrac{\partial \sigma_{t,k}^2}{\partial a_i^{(t-m)}} = \dfrac{2}{nk}\dsp\sum_{j=0}^m\left(a_i^{(t-j)}-\mu_{t,k} \right) \dfrac{\partial a_i^{(t-j)}}{\partial a_i^{(t-m)}} - \sum_{j=0}^{k-1}\sum_{s=1}^n\left(a_s^{(t-j)}-\mu_{t,k}\right) \dfrac{\partial \mu_{t,k}}{\partial a_i^{(t-m)}}
    $}.
\end{equation}
\end{prop}
Note that the computations of $ \dfrac{\partial y_i^{(t)}}{\partial \beta}$ and $\dfrac{\partial y_i^{(t)}}{\partial \gamma}$ are straightforward and are omitted.

In our experiments, we will use ATN on LSTM networks. Following~\cite{ba2016layer} and~\cite{CooijmansBLGC17}, our ATN method for LSTM is as follows :

\begin{align}
    \begin{pmatrix}
    f^{(t)} \\i^{(t)} \\ o^{(t)} \\ g^{(t)}
    \end{pmatrix} &= ATN(W_h h^{(t-1)}) + ATN(W_x x^{(t)}) + b \\
    c^{(t)} &= \sigma(f^{(t)})\odot c^{(t-1)} + \sigma(i^{(t)}) \odot \tanh(g^{(t)}) \\
    h^{(t)} &= \sigma(o^{(t)})\odot \tanh(ATN(c^{(t)}))
\end{align}
where $\odot$ is the Hadamard product and $\sigma(\cdot)$ is the sigmoid function. 

\section{Experiments}\label{section:experiments}

We have performed a series of experiments which include the Copying~\cite{hochreiter1997long}, Adding~\cite{hochreiter1997long}, and Denoise problems~\cite{8668730, https://doi.org/10.48550/arxiv.1611.09434} as well as Language Modeling on character level Penn Treebank dataset~\cite{10.5555/972470.972475} and word level WikiText-2 dataset~\cite{merity2016pointer}.

All experiments were run using Python 3.7.0, PyTorch 1.1.0, and CUDA 9.0 on a single NVIDIA Tesla V100 GPU.

\subsection{Synthetic Tasks}
\subsubsection{Copying} \label{exp:copying_section}

The copying problem is a common synthetic task that is used to test RNNs, which was originally proposed in~\cite{hochreiter1997long}. For this problem, a string of 10 digits is fed into the RNN sampled uniformly from the integers between 1 and 8. A sequence of $T$ zeros follows this, and a 9, marking the start of a string of 9 zeros, for a total length of $T + 20$. The objective of the task is to output the initial string of 10 digits beginning at the marker's location, copying the initial string from the front to the back. Cross-entropy loss is used to evaluate this model, with a baseline expected cross-entropy of $\frac{10\log{(8)}}{T+20}$ which represents selecting digits 1-8 at random after the 9. 

\textit{Implementation Details:} The models were trained with a batch size of 128, a single LSTM layer with a hidden size of 68, an RMSProp~\cite{tieleman2012lecture} optimizer with a learning rate of $10^{-4}$, and $T$ values of 100 and 200. The ATN model is implemented with $k=45$ for both $T$ values.

\begin{figure}[!tbp]
    \centering
    \begin{subfigure}{0.75\textwidth}
        \centering
        \includegraphics[width=.99\columnwidth]{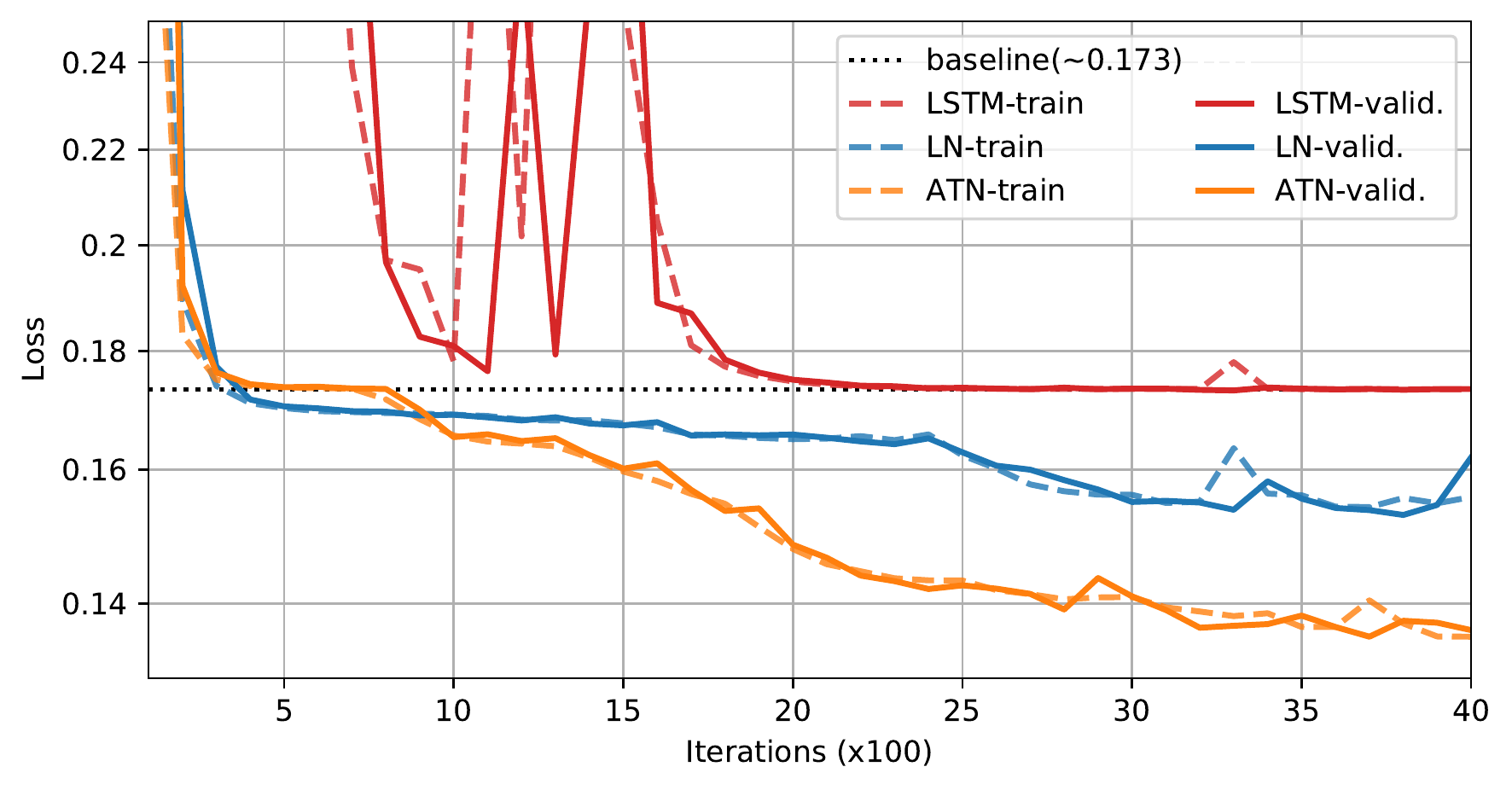}
        \caption{Copying problem with $T=100$}
        \label{exp:copy100}
    \end{subfigure}
    \hfill
    \begin{subfigure}{0.75\textwidth}
        \centering
        \includegraphics[width=.99\columnwidth]{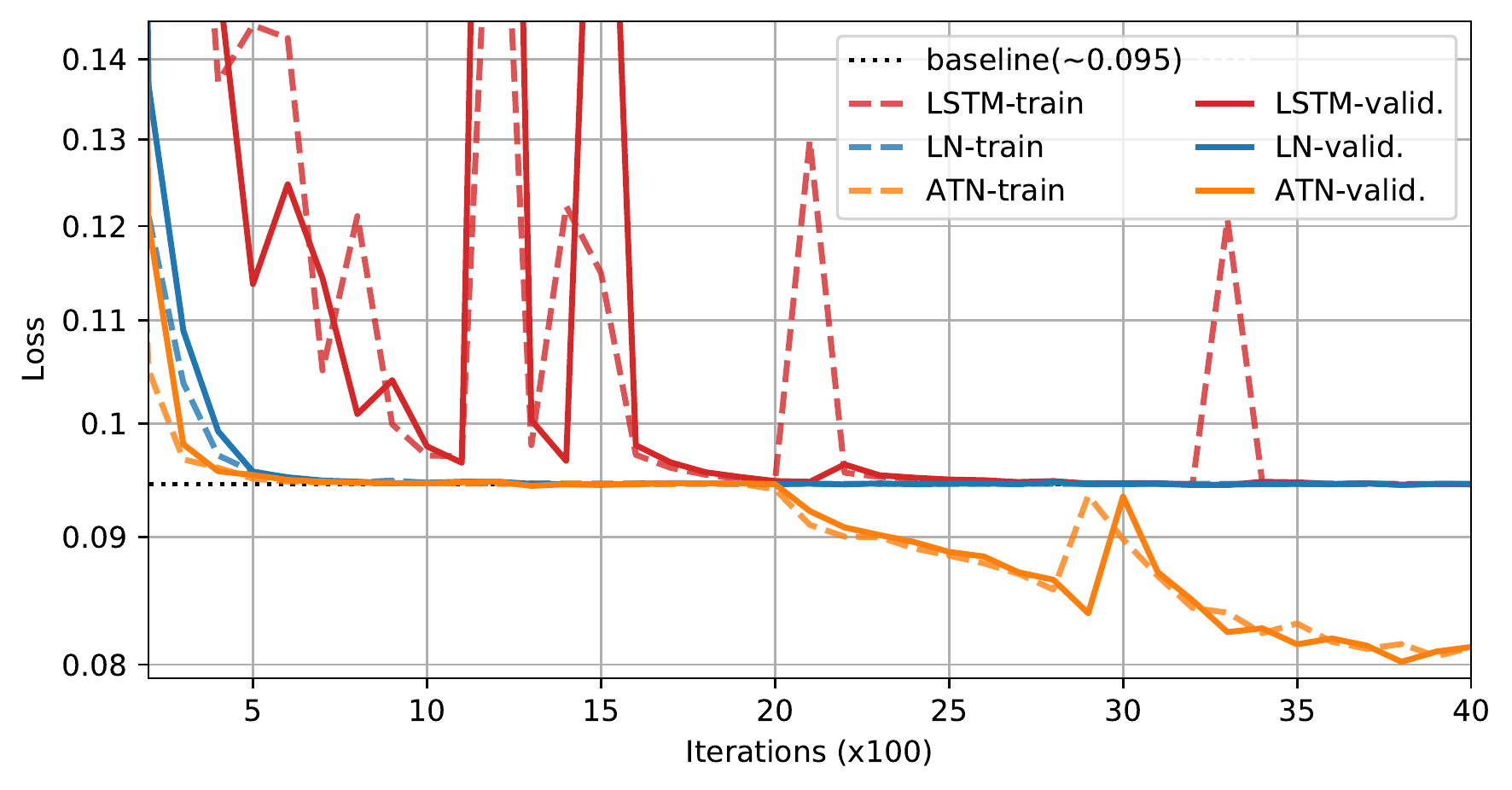}
        \caption{Copying problem with $T=200$}
        \label{exp:copy200}
    \end{subfigure}
    \caption{Results on the Copying problem for $T=100$ and $T=200$. The LSTM and LN-LSTM (LN) models are provided here for comparison purposes with our ATN-LSTM (ATN) method.}
\end{figure}

\begin{table}[!tbp]
    \centering
    \begin{tabular}{c|c|c||c|c}
        & \multicolumn{2}{c||}{Loss $\times$ $10^{-1}$ $\downarrow$} & \multicolumn{2}{c}{Loss $\times$ $10^{-2}$ $\downarrow$} \\
        \midrule
        Sequence Length & \multicolumn{2}{c||}{$T=100$} & \multicolumn{2}{c}{$T=200$} \\
        \midrule
        & train & validation & train & validation \\
        \toprule
        LSTM & 1.739 & 1.731 & 9.445 & 9.445\\
        \midrule
        LN & 1.542 & 1.529 & 9.445 & 9.445\\
        \midrule
        ATN & 1.354 & 1.354 & 8.020 & 8.020\\
        \bottomrule
    \end{tabular}
    \caption{Copying Results: Attained minimum values. $\downarrow$ - denotes the smaller, the better result.}
    \label{tab:copying}
\end{table}

\textit{Results:} For each of the sequence lengths tested, the plain LSTM is incapable of achieving losses below the baseline. While the LN-LSTM is able to do so to some extent on the $T = 100$ version, see Figure \ref{exp:copy100}, it also gets stuck at the baseline loss on the $T = 200$ task, Figure \ref{exp:copy200}. For both of these tasks, our ATN-LSTM model demonstrates eventual losses below those reached by the LN-LSTM model, Figures \ref{exp:copy100} and \ref{exp:copy200}. We also note that the initial rate of convergence is at least as steep if not steeper than that of the LN-LSTM model, demonstrating that the ATN-LSTM has a positive contribution to training in both the short and long term.

\subsubsection{Adding}\label{exp:adding_section}
The adding problem is another synthetic task for RNNs proposed in~\cite{hochreiter1997long}. Our implementation of this problem is a variation of the original problem. The RNN takes a 2-dimensional input of length T. The first dimension consists of a sequence of zeros except for two ones placed randomly in the first and second half of the sequence. The second dimension is a sequence of numbers selected uniformly from $[0, 1)$. The goal of the task is to take the numbers from the second dimension in positions corresponding to the ones and to output their sum.

\textit{Implementation Details:}  The models were trained with a batch size of 50, a single LSTM layer with a hidden size of 60, and an RMSprop~\cite{tieleman2012lecture} optimizer with a learning rate of $10^{-3}$. We use $T$ values of 100 and 200. This task is evaluated with a mean-squared error. The ATN model is implemented with $k$ values of 25 for $T=100$ and 5 for $T=200$.

\textit{Results:} Our model shows consistent improvement over the LSTM and LN-LSTM models. For each example, the ATN shows a rapid initial convergence before settling into a slower rate which is roughly parallel to that of the LN-LSTM. In Figure \ref{exp:add100}, this initial conversion almost manages to take the model to the same loss as is achieved by the LN-LSTM after the entirety of the training. In Figure \ref{exp:add200}, the LN-LSTM is able to separate itself further from the LSTM than in Figure \ref{exp:add100} but is still at a higher loss than the ATN for all but the very beginning of training.

\begin{figure}[!tbp]
    \centering
    \begin{subfigure}{0.75\textwidth}
        \centering
        \includegraphics[width=.99\columnwidth]{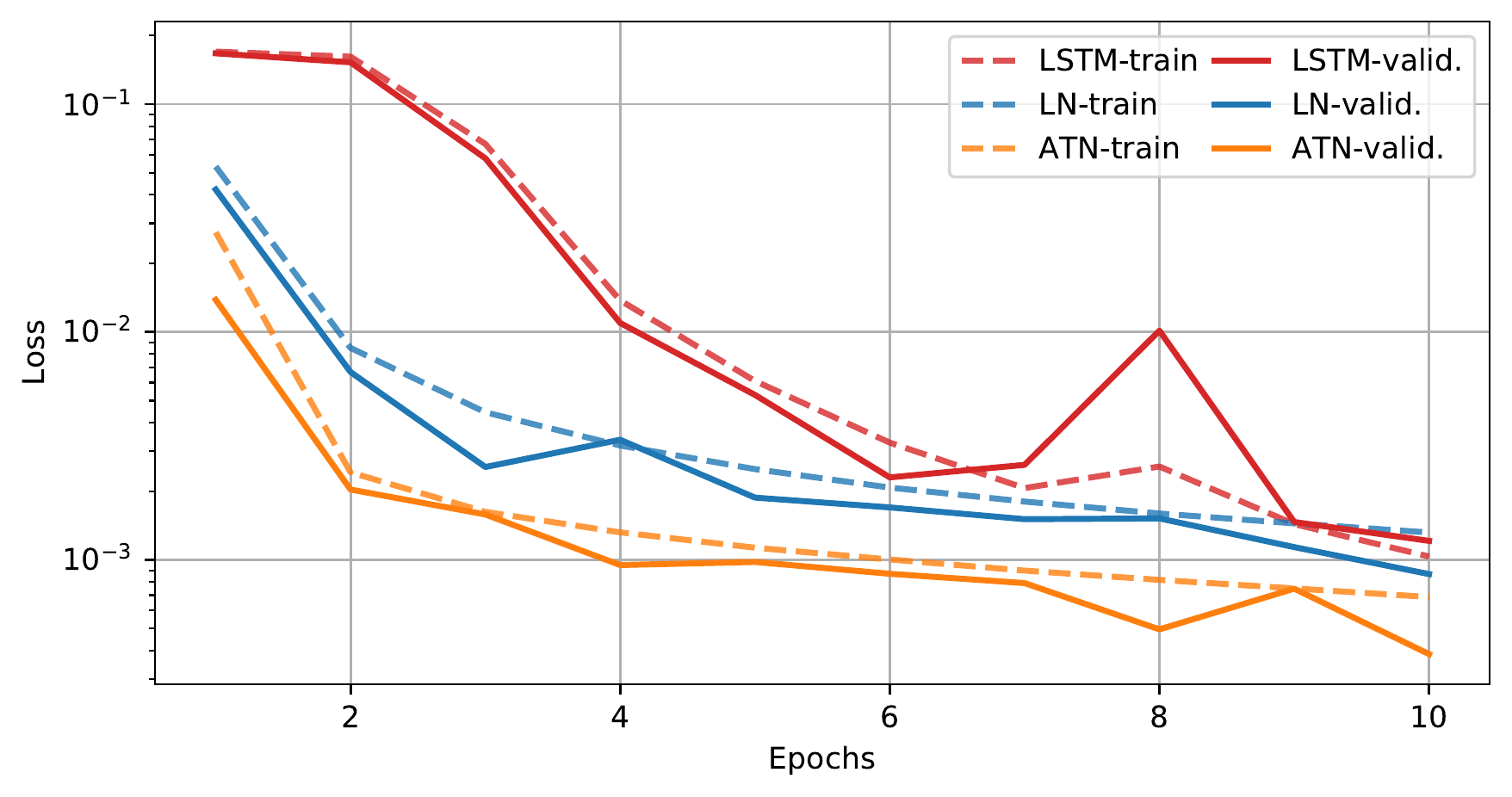}
        \caption{Adding problem with $T=100$}
        \label{exp:add100}
    \end{subfigure}
    \hfill
    \begin{subfigure}{0.75\textwidth}
        \centering
        \includegraphics[width=.99\columnwidth]{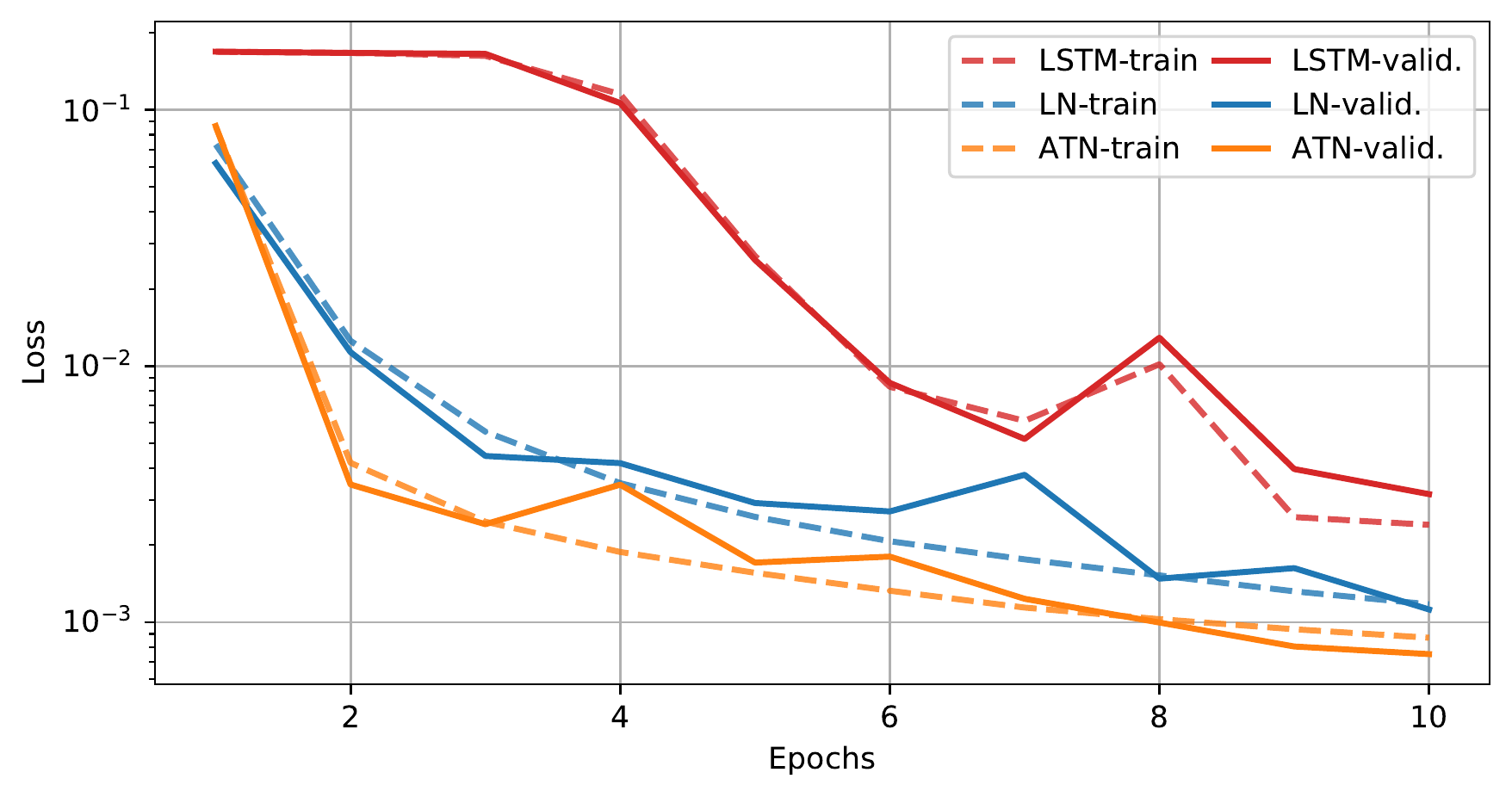}
        \caption{Adding problem with $T=200$}
        \label{exp:add200}
    \end{subfigure}
    \caption{Results on the Adding problem for $T=100$ and $T=200$.}
\end{figure}

\begin{table}[!tbp]
    \centering
    \begin{tabular}{c|c|c||c|c}
        & \multicolumn{2}{c||}{Loss $\times 10^{-3}$ $\downarrow$} & \multicolumn{2}{c}{Loss $\times 10^{-3}$ $\downarrow$} \\
        \midrule
        Sequence Length & \multicolumn{2}{c||}{$T=100$} & \multicolumn{2}{c}{$T=200$} \\
        \midrule
        & train & validation & train & validation \\
        \toprule
        LSTM & 1.034 & 1.212 & 2.390 & 3.161\\
        \midrule
        LN & 1.319 & 0.866 & 1.174 & 1.121\\
        \midrule
        ATN & 0.687 & 0.385 & 0.869 & 0.750\\
        \bottomrule
    \end{tabular}
    \caption{Adding Results: Attained minimum values. $\downarrow$ - denotes the smaller, the better result.}
    \label{tab:adding}
\end{table}

\subsubsection{Denoise Task}\label{denoiseproblem}

\begin{figure}[!tbp]
    \centering
    \begin{subfigure}{0.75\textwidth}
        \centering
        \includegraphics[width=.99\columnwidth]{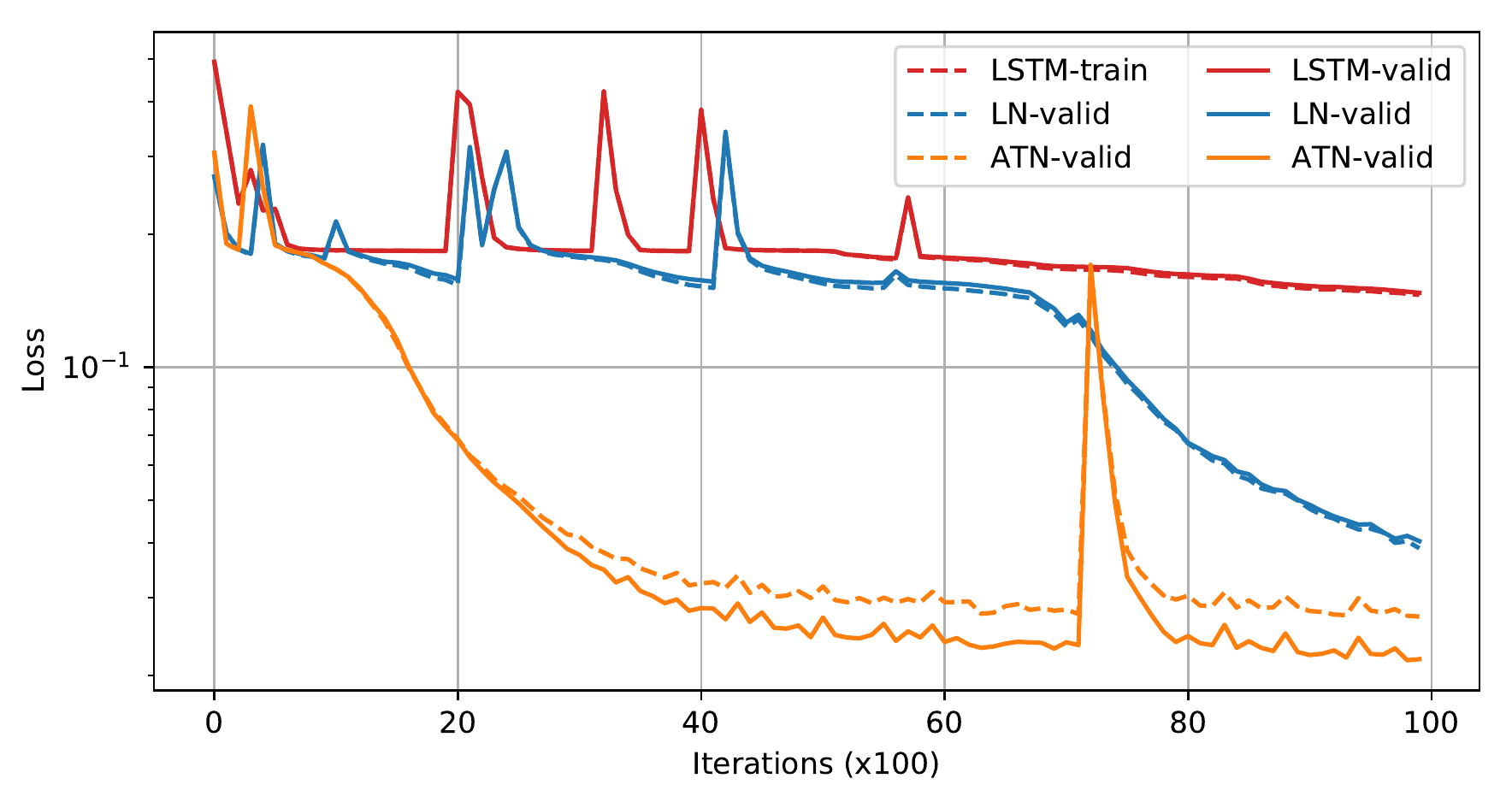}
        \caption{Denoise task with $T=100$}
        \label{exp:denoise100}
    \end{subfigure}
    \hfill
    \begin{subfigure}{0.75\textwidth}
        \centering
        \includegraphics[width=.99\columnwidth]{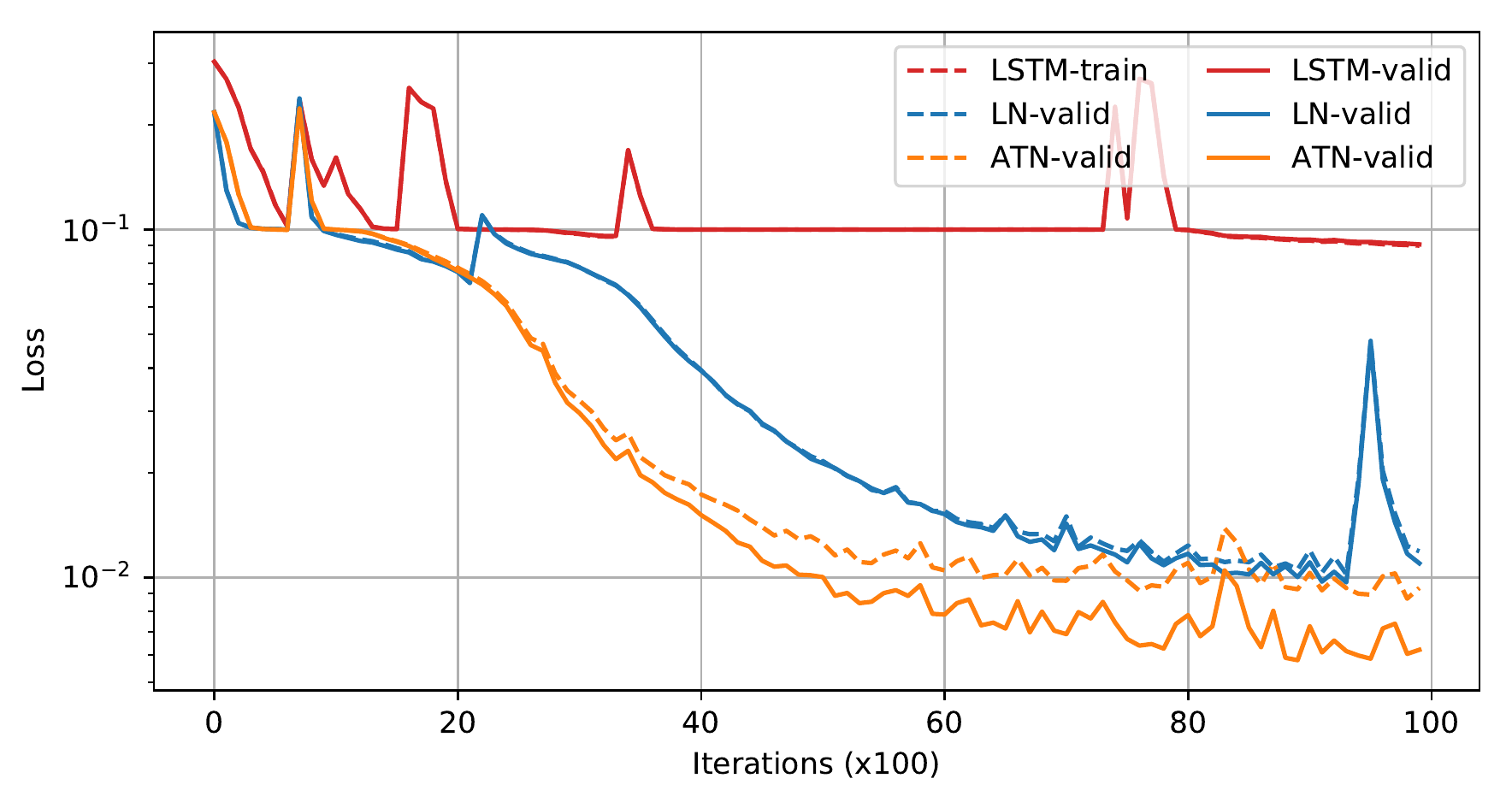}
        \caption{Denoise task with $T=200$}
        \label{exp:denoise200}
    \end{subfigure}
    \caption{Results on the Denoise task for $T=100$ and $T=200$.}
\end{figure}

\begin{table}[!tbp]
    \centering
    \begin{tabular}{c|c|c||c|c}
        & \multicolumn{2}{c||}{Loss $\times 10^{-2}$ $\downarrow$} & \multicolumn{2}{c}{Loss $\times 10^{-3}$ $\downarrow$} \\
        \midrule
        Sequence Length & \multicolumn{2}{c||}{$T=100$} & \multicolumn{2}{c}{$T=200$} \\
        \midrule
        & train & validation & train & validation \\
        \toprule
        LSTM & 14.22 & 14.64 & 88.30 & 89.88\\
        \midrule
        LN & 2.489 & 3.169 & 6.108 & 9.057\\
        \midrule
        ATN & 1.733 & 2.073 & 4.70 & 5.433\\
        \bottomrule
    \end{tabular}
    \caption{Denoise Results: Attained minimum values. $\downarrow$ - denotes the smaller, the better result.}
    \label{tab:denoise}
\end{table}

The Denoise Task~\citep{8668730, https://doi.org/10.48550/arxiv.1611.09434} is another synthetic problem that requires filtering out the noise out of a noisy sequence. This problem requires the forgetting ability of the network as well as learning long-term dependencies coming from the data~\citep{8668730}. The input sequence of length $T$ contains 10 randomly located data points, and the other $T-10$ points are considered noise data. These 10 points are selected from a dictionary $\dsp\{a_i\}_{i=0}^{n+1}$, where the first $n$ elements are data points, and the other two are the $``noise"$ and the $``marker"$ respectively. The output data consists of the list of the data points from the input, and it should be outputted as soon as it receives the $``marker"$. The goal is to filter out the noise and output the random 10 data points chosen from the input. 

\emph{Implementation Details:}
The models were trained using a batch size of 128, a single LSTM layer with a hidden size of 100, and Adam~\cite{kingma2014adam} optimizer with a learning rate of $10^{-2}$. We use $T$ values of 100 and 200. The ATN model is implemented with $k$ values of 20 and 60 for $T=100$ and $T=200$ respectively.

\emph{Results:} For both sequence lengths, our models outperform the LSTM and the LN-LSTM throughout training. While the LN-LSTM model can surpass the baseline set by the LSTM, it does so later than the ATN model, and its convergence curve flattens out at a higher loss than the ATN model.

\subsection{Language Models}
Language modeling is one of many natural language processing tasks. It is the development of probabilistic models that are capable of predicting the next word or character in a sequence using information that has preceded it. For both of the Language Modeling problems, we based our experiments on the AWD-LSTM model~\cite{merity2018regularizing}. 

\subsubsection{Character Level Penn Treebank}

The models were tested on their suitability for language modeling tasks using the character level Penn Treebank dataset~\cite{10.5555/972470.972475} also known as character-PTB or simply cPTB dataset. This dataset is a collection of English-language Wall Street Journal articles. The dataset consists of a vocabulary of 10,000 words with other words replaced as \verb|<unk>|, resulting in approximately 6 million characters that are divided into 5.1 million, 400 thousand, and 450 thousand character sets for training, validation, and testing, respectively with a character alphabet size of 50. The goal of the character-level Language Modeling task is to predict the next character given the preceding sequence of characters.

\textit{Implementation Details:} For this task, we partitioned the training sequence into 220 character length subsequences. The models were trained using a batch size of 32, a single LSTM layer with a hidden size of 1,000, an Adam~\cite{kingma2014adam} optimizer with a learning rate of $10^{-2}$, gradient clipping by norm at 3, and learning rate decay by a factor of 10 at epoch 80 and 90. The ATN model is implemented with a $k$ value of 10.

\textit{Results:} Our model shows improvement over the LSTM and the LN-LSTM models, the comparison results are presented in Table \ref{tab:ptbc}.

\begin{table}
    \centering
    \begin{tabular}{c|c|c}
        & \multicolumn{2}{c}{bpc $\downarrow$} \\
        \toprule
        & train & validation \\
        \midrule
        LSTM & 1.692 & 1.743\\
        \midrule
        LN & 1.390 & 1.520 \\
        \midrule
        ATN & 1.381 & 1.511 \\
        \bottomrule
    \end{tabular}
    \caption{Character Level Penn Treebank Results: Attained minimum values. $\downarrow$ - denotes the smaller, the better result}
    \label{tab:ptbc}
\end{table}

\subsubsection{WikiText-2}

\begin{table}
    \centering
    \begin{tabular}{c|c|c}
        & \multicolumn{2}{c}{PPL $\downarrow$} \\
        \toprule
        & train & validation \\
        \midrule
        LSTM & 80.68 & 65.65 \\
        \midrule
        LN & 80.24 & 58.0 \\
        \midrule
        ATN & 78.55 & 56.06 \\
        \bottomrule
    \end{tabular}
    \caption{WikiText-2 Results: Attained minimum values. $\downarrow$ - denotes the smaller, the better result.}
    \label{tab:wt2}
\end{table}

The WikiText-2 dataset was introduced in~\cite{merity2016pointer}. It is approximately two times the size of the Penn Treebank dataset and contains preprocessed Wikipedia articles while maintaining the original structure, punctuation, and symbols. The WikiText-2 dataset consists of approximately 2.2 million words: 2 million for the training set and 200 thousand for the validation and test sets, with a vocabulary size of 33,278. This task is a word-level Language Modeling problem with the goal to predict the next word given the preceding sequence of words.

\textit{Implementation Details:} We used a batch size of 32; three LSTM layers with embedding and hidden sizes of 400 and 1,150, respectively; BPTT values of 70; gradient clipping on the norm of 0.25; and learning rate of 30 with Stochastic Gradient Descent (SGD) optimizer without any momentum or learning rate decay, and switch to ASGD~\cite{10.1137/0330046} optimizer using nonmono criteria from~\cite{merity2018regularizing} with value 5 (our experiments showed that switching happens approximately between epochs 20 and 30 for all models: LSTM, LN, and ATN). The ATN model is implemented with a $k$ value of $25$.

\textit{Results:} In this experiment, the ATN method shows improvement over LSTM and LN method in both training and validation perplexity (PPL), see Table \ref{tab:wt2}.

\section{Ablation Studies}\label{section:ablation_studies}

\subsection{Input statistic invariance across time}\label{abl_mnist}
\begin{figure}[!t]
    \centering
    \includegraphics[width=.75\columnwidth]{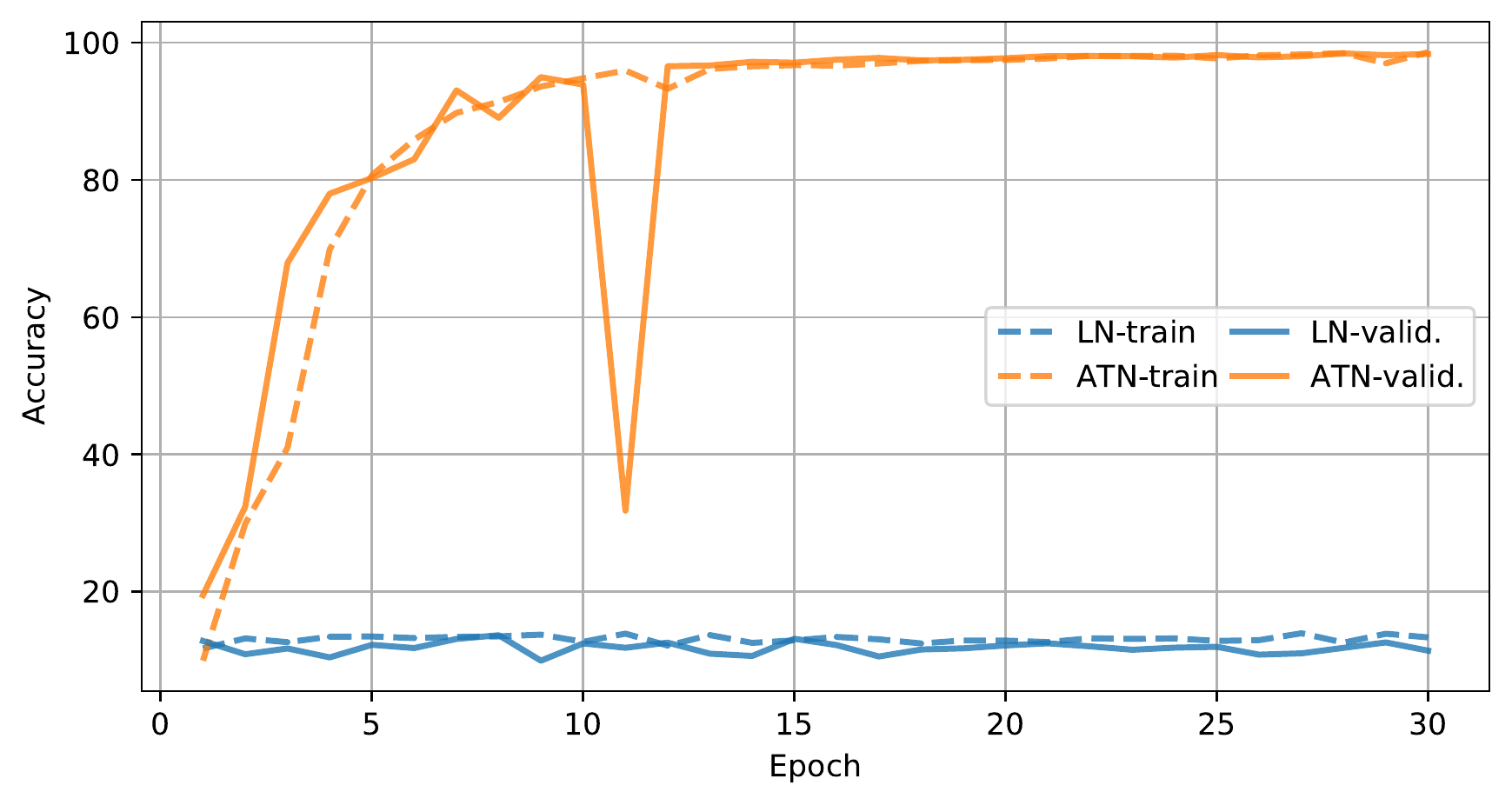}
    \caption{Pixel-by-pixel MNIST}
    \label{abl_st:mnist}
\end{figure}

In most implementations of LN-LSTM, including the one used in the experiments above, the inputs to the normalization method are the results of a linear layer, including both weight and bias. This differs slightly from the model proposed in ~\cite{ba2016layer} in that their version placed the LSTM bias outside of the normalization. Using that original architecture, we can clearly demonstrate the underlying problem with Layer Normalization that we aim to solve, the loss of input information, by setting the statistics to constant values across time. 

To show this, we use the MNIST dataset~\cite{726791} after applying Gaussian noise with variance $0.1$, for the pixel-by-pixel task~\cite{le2015simple}. This task takes the pixel values of a handwritten digit and inputs them as an unpermuted sequence of length 784 in order to predict the digit class. Due to the high probability of pixels having near zero values, we needed to use $\varepsilon$ values of $1$ in both normalization schemes. With this task, we can see in Figure \ref{abl_st:mnist} that the use of Layer Normalization renders the model completely incapable of training. Because LN takes the information from each pixel and normalizes it to the exact same distribution, it erases everything the model could use to learn, making it no better than guessing. The ATN method with $k=10$ solves this problem by its use of multiple time steps in calculating the mean and variance, meaning that the normalized outputs will not all have identical statistics. This change allows ATN to perform quite well, even when Layer Normalization cannot.

\subsection{Post Normalization Statistics}

In Figure \ref{abd_st}, we present the statistics of the post normalization components from a single iteration of training for the Adding Problem~\cite{hochreiter1997long} described in Section \ref{exp:adding_section} with $T=75$. We present the statistics from four different models, an LN-LSTM, and three ATN($k$) models with $k$ values of $5$, $25$, and $55$. All of the models did not include the use of trainable bias and gain parameters inside the normalization methods. 

In Figure \ref{abd_st:postnorm_hh}, we show the mean and variance after normalization of the product of the hidden-to-hidden weight and the hidden state, $W_{h} h^{(t-1)}$. While Layer Normalization produces constant mean and variance, the ATN method allows for the statistics to vary at each time step, resulting in curves that do not differ too much from those for LN in terms of scale but do demonstrate the natural fluctuations in the hidden states. From this, we can see that we are achieving the combination of a controlled output that is still capable of reflecting the temporal changes of the network. 

In Figure \ref{abd_st:postnorm_ih}, we show the statistics from the product of the input-to-hidden weight and the input, $W_{x} x^{(t)}$. The ATN model provides highly variable means and variances, showcasing the amount of information about the dataset which is lost when LN resets the statistics to these constant values. 

In Figure \ref{abd_st:postnorm_c}, we show the post normalization statistics of the memory cell, $c^{(t)}$. These statistics clearly demonstrate the effect of a shorter $k$ value as opposed to a longer one in the mean. In the early iterations for the $k = 5$ model, the mean has a larger spike which flattens to a bit above zero by the end of the iteration. For the larger $k$ values, this initially increased mean gets maintained throughout a larger portion of the iteration, causing the lower values further along to have less influence on the statistics.

\begin{figure}
    \centering
    \begin{subfigure}{0.99\textwidth}
        \centering
        \includegraphics[height=2.25in]{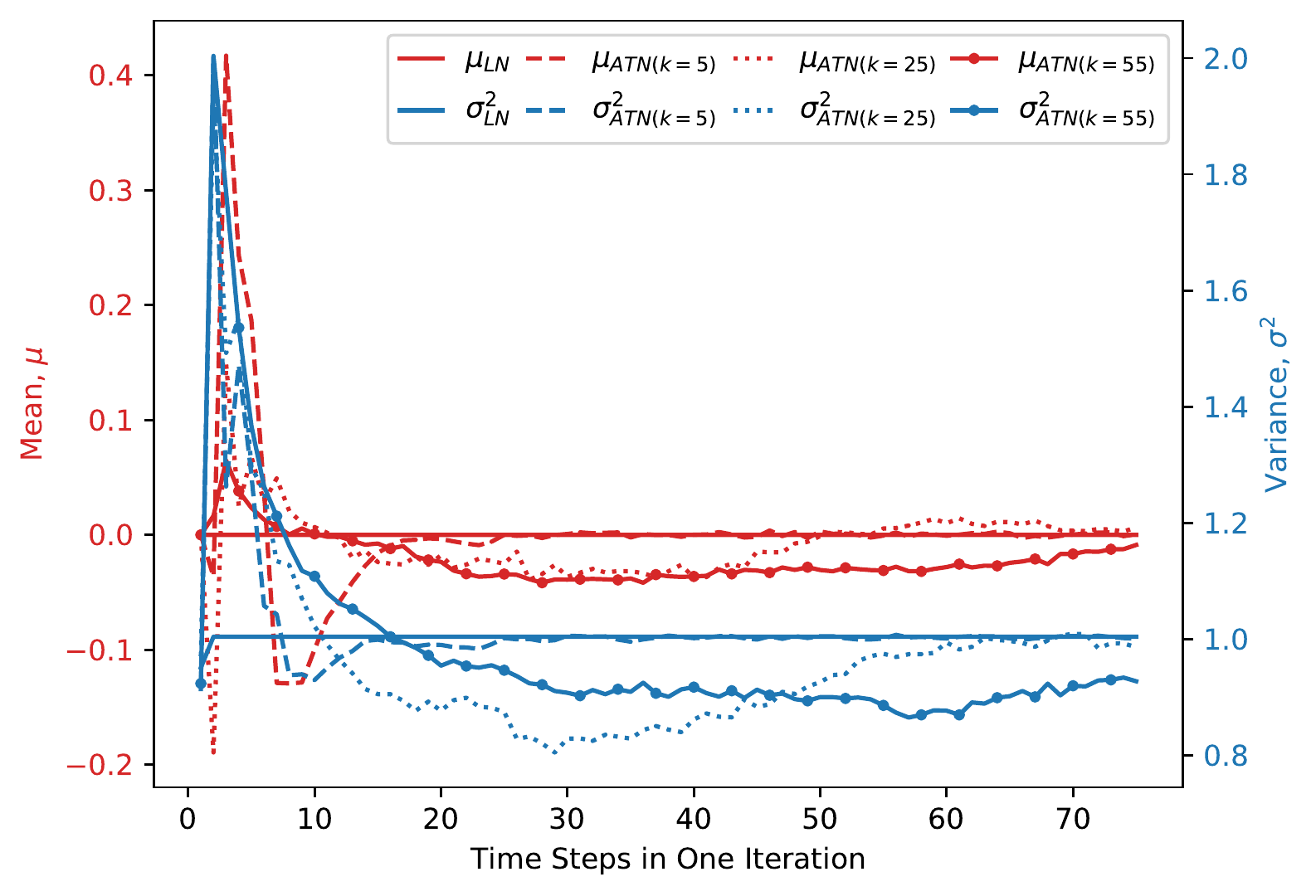}
        \caption{Hidden-to-Hidden}
        \label{abd_st:postnorm_hh}
    \end{subfigure}
    \begin{subfigure}{0.99\textwidth}
        \centering
        \includegraphics[height=2.25in]{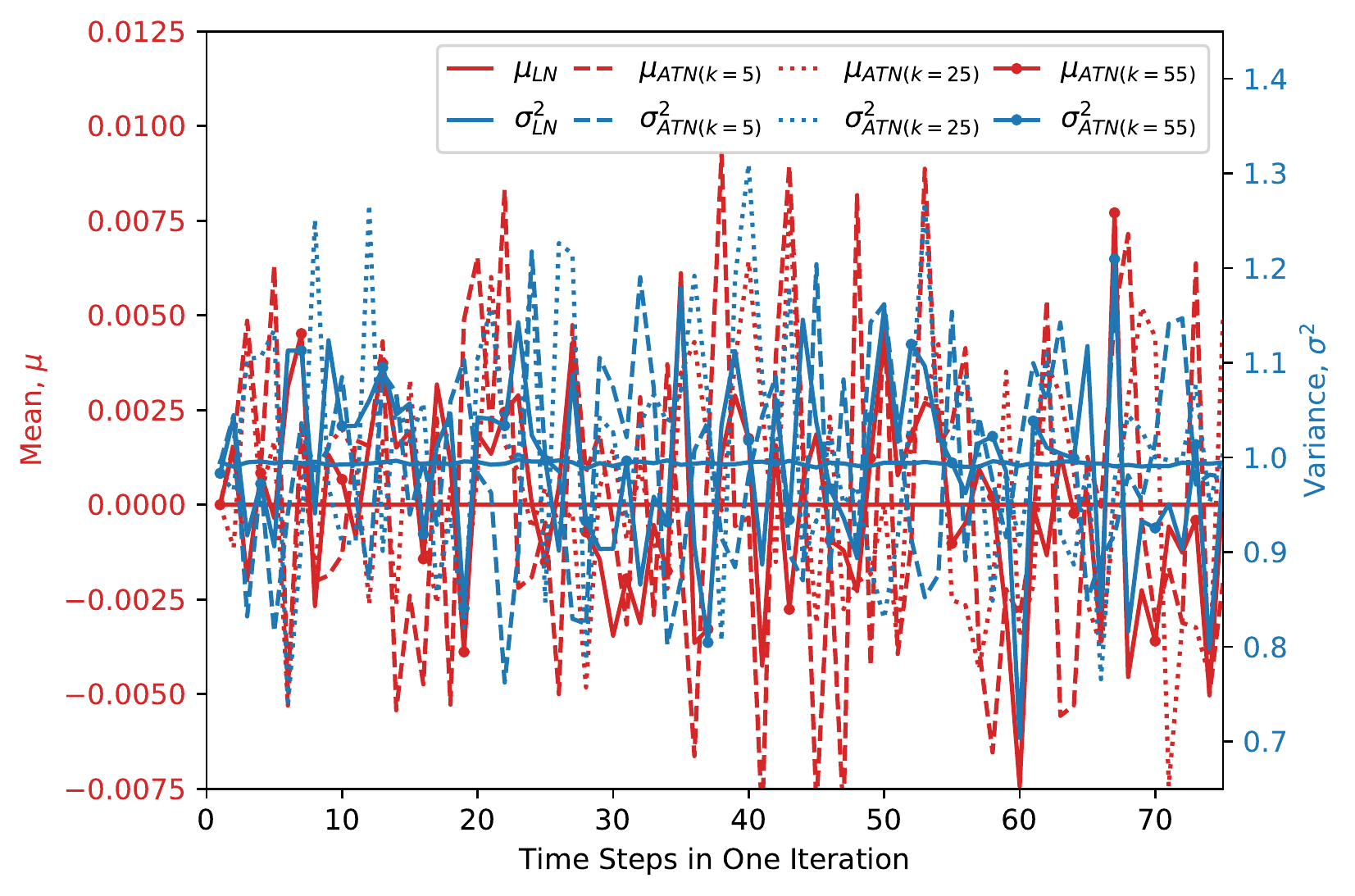}
        \caption{Input-to-Hidden}
        \label{abd_st:postnorm_ih}
    \end{subfigure}
    \begin{subfigure}{0.99\textwidth}
        \centering
        \includegraphics[height=2.25in]{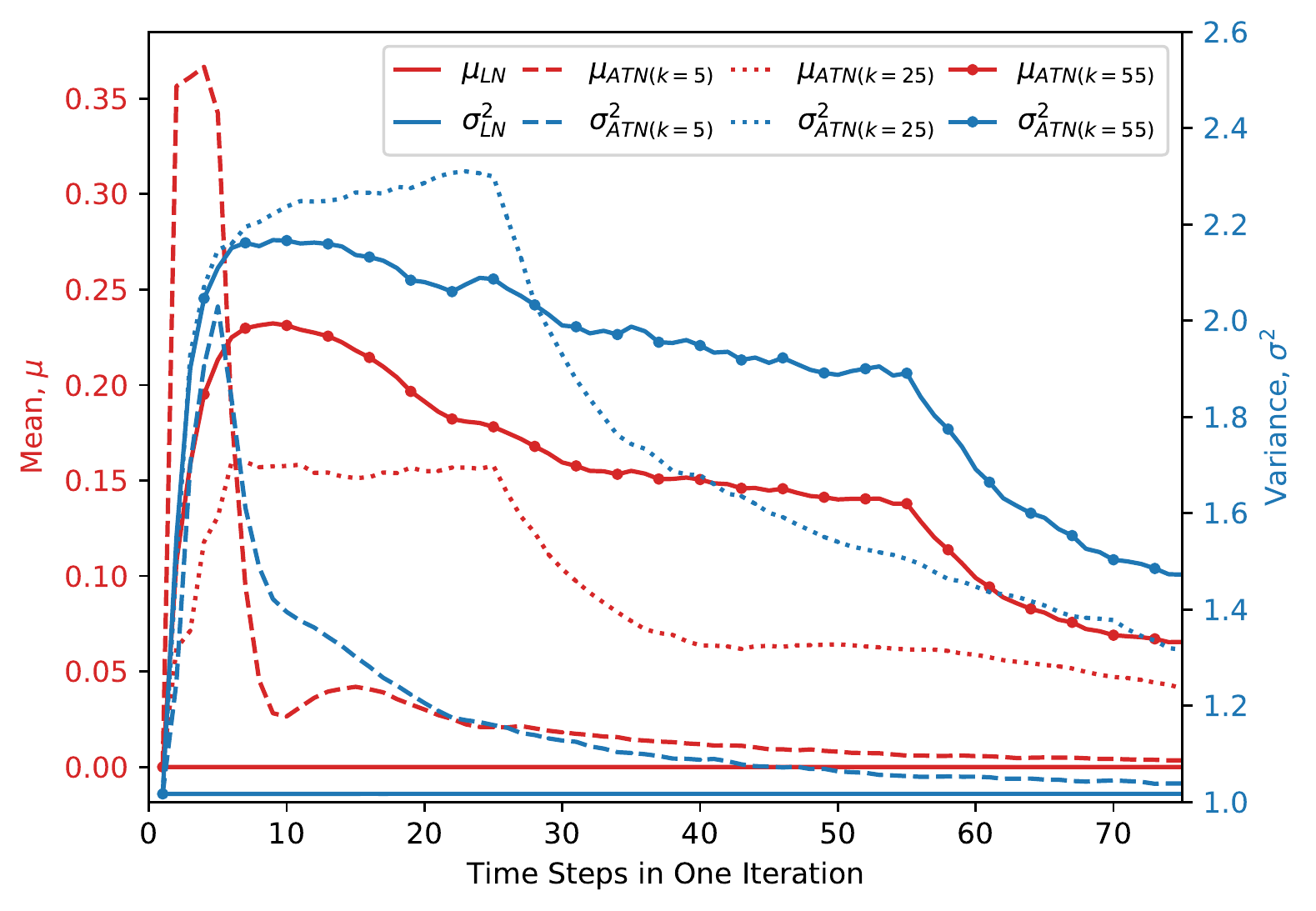}
        \caption{Memory Cell}
        \label{abd_st:postnorm_c}
    \end{subfigure}
    \caption{Post Normalization Statistics for Adding Problem with $T=75$}
    \label{abd_st}
\end{figure}

\subsection{Optimal $k$ Value for ATN method}

To highlight the importance of normalizing with respect to $k$ time steps instead of just one or all of them, we present a study on various $k$ values. In Figure \ref{abl_st:1}, we present results on the Copying Problem~\cite{hochreiter1997long} described in Section \ref{exp:copying_section} with $T=100$. For this experiment, we have trained LSTM, LN, and three ATN($k$) models with values of $k$ being 25, 45, and 65 under the same conditions. 

All ATN models perform better than both LSTM and LN. The ATN($k=45$) model performs better than ATN($k=25$) which should not be a surprise since the larger $k$ value would mean we are normalizing with respect to a larger set and getting better statistics for the mean and variance, however, ATN($k=65$) performs poorer than ATN($k=45$) and even poorer than ATN($k=25$) which suggests that too large $k$ may actually degrade the result. This may be due to numerical difficulties in propagating derivative through $k$ steps in ATN for a large $k$.

\begin{figure}
    \centering
    \includegraphics[width=.75\textwidth]{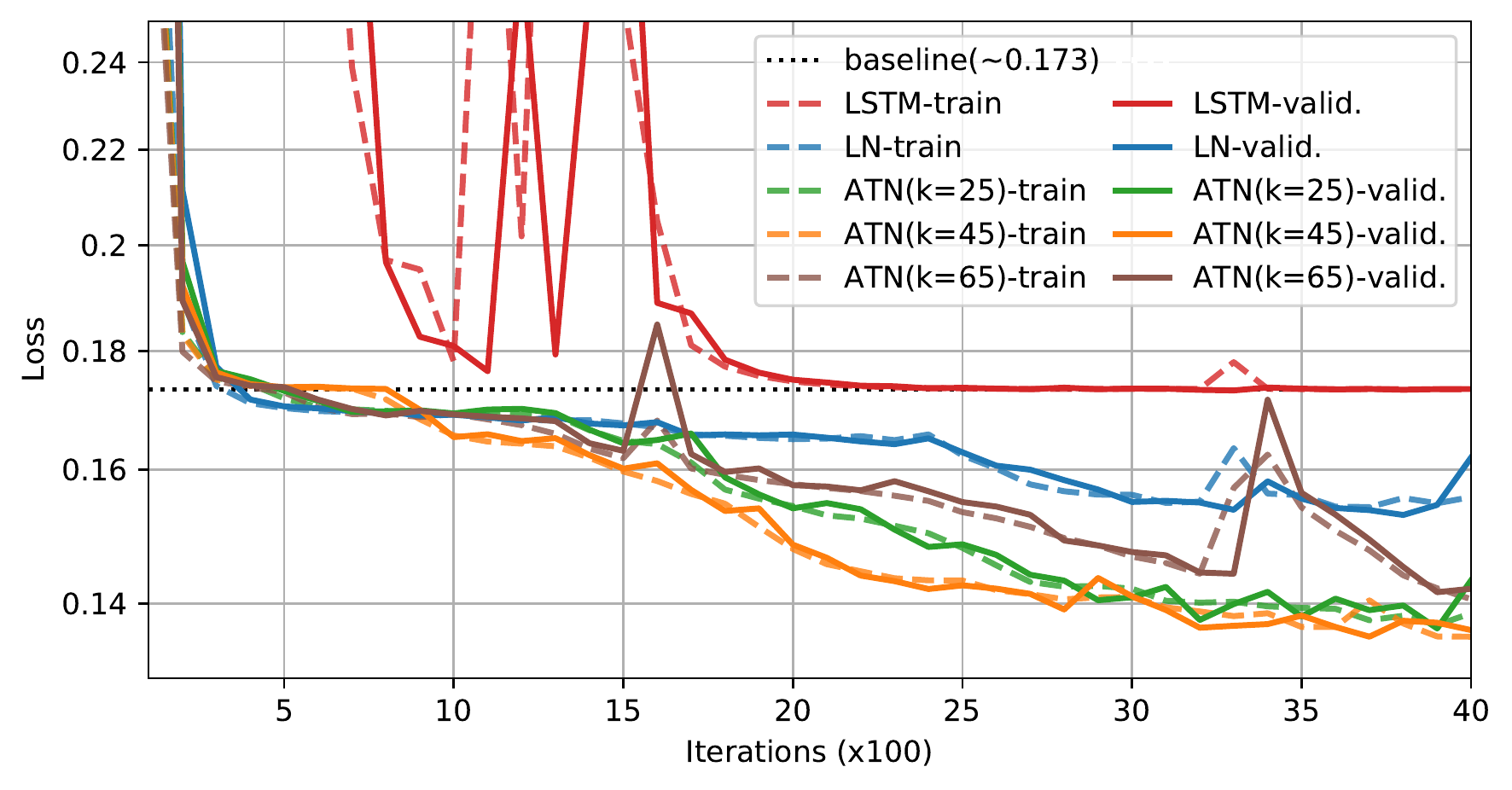}
    \caption{$k$ value study in ATN method}
    \label{abl_st:1}
\end{figure}

\section{Conclusion}\label{section:conclusion}
In this paper, we have introduced a method for adapting statistics-based normalization methods to recurrent neural networks to break the time invariance of the traditional normalization methods. We have presented theoretical results on the impact this method has on the model's gradients, as well as showing the preservation of invariance to the rescaling of the weight matrix. Our experiments demonstrate that our ATN-LSTM improves over LN for LSTM in both training and testing results. In light of the popularity of LN in practical applications, our method offers an important alternative for further improving RNN performance.

\section*{Acknowledgment}
We would like to thank the University of Kentucky Center for Computational Sciences and Information Technology Services Research Computing for their support and use of the Lipscomb Compute Cluster and associated research computing resources.

\section*{Ethical statement}

An idea that we proposed in this manuscript uses theoretical and experimental methods to develop improved normalization for training Recurrent Neural Networks. This work improves a well-known Layer Normalization technique that is widely used in various deep learning architectures and their applications. However, the method proposed in this work can be applied to other normalization methods such as BN for LSTMs. This work has no ethical or future societal consequence outside of the usage of it by unknown to us parties on unpredictable applications.


\bibliographystyle{elsarticle-num} 
\bibliography{refs}

\begin{thebibliography}{10}
\expandafter\ifx\csname url\endcsname\relax
  \def\url#1{\texttt{#1}}\fi
\expandafter\ifx\csname urlprefix\endcsname\relax\def\urlprefix{URL }\fi
\expandafter\ifx\csname href\endcsname\relax
  \def\href#1#2{#2} \def\path#1{#1}\fi

\bibitem{10.5555/104279.104293}
D.~E. Rumelhart, G.~E. Hinton, R.~J. Williams, Learning Internal
  Representations by Error Propagation, MIT Press, 1986, p. 318–362.

\bibitem{doi:10.1073/pnas.79.8.2554}
J.~J. Hopfield,
  \href{https://www.pnas.org/doi/abs/10.1073/pnas.79.8.2554}{Neural networks
  and physical systems with emergent collective computational abilities.},
  Proceedings of the National Academy of Sciences 79~(8) (1982) 2554--2558.
\newblock \href
  {http://arxiv.org/abs/https://www.pnas.org/doi/pdf/10.1073/pnas.79.8.2554}
  {\path{arXiv:https://www.pnas.org/doi/pdf/10.1073/pnas.79.8.2554}}, \href
  {https://doi.org/10.1073/pnas.79.8.2554} {\path{doi:10.1073/pnas.79.8.2554}}.
\newline\urlprefix\url{https://www.pnas.org/doi/abs/10.1073/pnas.79.8.2554}

\bibitem{Hochreiter:1997:LSM:1246443.1246450}
S.~Hochreiter, J.~Schmidhuber,
  \href{http://dx.doi.org/10.1162/neco.1997.9.8.1735}{Long short-term memory},
  Neural Comput. 9~(8) (1997) 1735--1780.
\newblock \href {https://doi.org/10.1162/neco.1997.9.8.1735}
  {\path{doi:10.1162/neco.1997.9.8.1735}}.
\newline\urlprefix\url{http://dx.doi.org/10.1162/neco.1997.9.8.1735}

\bibitem{https://doi.org/10.48550/arxiv.1406.1078}
K.~Cho, B.~van Merrienboer, C.~Gulcehre, D.~Bahdanau, F.~Bougares, H.~Schwenk,
  Y.~Bengio, \href{https://arxiv.org/abs/1406.1078}{Learning phrase
  representations using rnn encoder-decoder for statistical machine
  translation} (2014).
\newblock \href {https://doi.org/10.48550/ARXIV.1406.1078}
  {\path{doi:10.48550/ARXIV.1406.1078}}.
\newline\urlprefix\url{https://arxiv.org/abs/1406.1078}

\bibitem{8668730}
L.~Jing, C.~Gulcehre, J.~Peurifoy, Y.~Shen, M.~Tegmark, M.~Soljacic, Y.~Bengio,
  Gated orthogonal recurrent units: On learning to forget, Neural Computation
  31~(4) (2019) 765--783.
\newblock \href {https://doi.org/10.1162/necoa01174}
  {\path{doi:10.1162/necoa01174}}.

\bibitem{https://doi.org/10.48550/arxiv.2208.06496}
E.~Mucllari, V.~Zadorozhnyy, C.~Pospisil, D.~Nguyen, Q.~Ye,
  \href{https://arxiv.org/abs/2208.06496}{Orthogonal gated recurrent unit with
  neumann-cayley transformation} (2022).
\newblock \href {https://doi.org/10.48550/ARXIV.2208.06496}
  {\path{doi:10.48550/ARXIV.2208.06496}}.
\newline\urlprefix\url{https://arxiv.org/abs/2208.06496}

\bibitem{ba2016layer}
J.~L. Ba, J.~R. Kiros, G.~E. Hinton, Layer normalization (2016).
\newblock \href {http://arxiv.org/abs/1607.06450} {\path{arXiv:1607.06450}}.

\bibitem{pmlr-v37-ioffe15}
S.~Ioffe, C.~Szegedy,
  \href{https://proceedings.mlr.press/v37/ioffe15.html}{Batch normalization:
  Accelerating deep network training by reducing internal covariate shift}, in:
  F.~Bach, D.~Blei (Eds.), Proceedings of the 32nd International Conference on
  Machine Learning, Vol.~37 of Proceedings of Machine Learning Research, PMLR,
  Lille, France, 2015, pp. 448--456.
\newline\urlprefix\url{https://proceedings.mlr.press/v37/ioffe15.html}

\bibitem{CooijmansBLGC17}
T.~Cooijmans, N.~Ballas, C.~Laurent, {\c{C}}.~G{\"{u}}l{\c{c}}ehre, A.~C.
  Courville, \href{https://openreview.net/forum?id=r1VdcHcxx}{Recurrent batch
  normalization}, in: 5th International Conference on Learning Representations,
  {ICLR} 2017, Toulon, France, April 24-26, 2017, Conference Track Proceedings,
  OpenReview.net, 2017.
\newline\urlprefix\url{https://openreview.net/forum?id=r1VdcHcxx}

\bibitem{ulyanov2017instance}
D.~Ulyanov, A.~Vedaldi, V.~Lempitsky, Instance normalization: The missing
  ingredient for fast stylization (2017).
\newblock \href {http://arxiv.org/abs/1607.08022} {\path{arXiv:1607.08022}}.

\bibitem{Wu_2018_ECCV}
Y.~Wu, K.~He, Group normalization, in: Proceedings of the European Conference
  on Computer Vision (ECCV), 2018.

\bibitem{NIPS2016_ed265bc9}
T.~Salimans, D.~P. Kingma, Weight normalization: A simple reparameterization to
  accelerate training of deep neural networks, in: D.~Lee, M.~Sugiyama,
  U.~Luxburg, I.~Guyon, R.~Garnett (Eds.), Advances in Neural Information
  Processing Systems, Vol.~29, Curran Associates, Inc., 2016.

\bibitem{gitman2017comparison}
I.~Gitman, B.~Ginsburg, Comparison of batch normalization and weight
  normalization algorithms for the large-scale image classification (2017).
\newblock \href {http://arxiv.org/abs/1709.08145} {\path{arXiv:1709.08145}}.

\bibitem{NEURIPS2019_2f4fe03d}
J.~Xu, X.~Sun, Z.~Zhang, G.~Zhao, J.~Lin, Understanding and improving layer
  normalization, in: H.~Wallach, H.~Larochelle, A.~Beygelzimer,
  F.~d\textquotesingle Alch\'{e}-Buc, E.~Fox, R.~Garnett (Eds.), Advances in
  Neural Information Processing Systems, Vol.~32, Curran Associates, Inc.,
  2019, p. 4381–4391.

\bibitem{hochreiter1997long}
S.~Hochreiter, J.~Schmidhuber, Long short-term memory, Neural computation 9~(8)
  (1997) 1735--1780.

\bibitem{https://doi.org/10.48550/arxiv.1611.09434}
J.~N. Foerster, J.~Gilmer, J.~Chorowski, J.~Sohl-Dickstein, D.~Sussillo,
  \href{https://arxiv.org/abs/1611.09434}{Input switched affine networks: An
  rnn architecture designed for interpretability} (2016).
\newblock \href {https://doi.org/10.48550/ARXIV.1611.09434}
  {\path{doi:10.48550/ARXIV.1611.09434}}.
\newline\urlprefix\url{https://arxiv.org/abs/1611.09434}

\bibitem{10.5555/972470.972475}
M.~P. Marcus, M.~A. Marcinkiewicz, B.~Santorini, Building a large annotated
  corpus of english: The penn treebank, Comput. Linguist. 19~(2) (1993)
  313–330.

\bibitem{merity2016pointer}
S.~Merity, C.~Xiong, J.~Bradbury, R.~Socher, Pointer sentinel mixture models
  (2016).
\newblock \href {http://arxiv.org/abs/1609.07843} {\path{arXiv:1609.07843}}.

\bibitem{tieleman2012lecture}
T.~Tieleman, G.~Hinton, Lecture 6.5-rmsprop: Divide the gradient by a running
  average of its recent magnitude, COURSERA: Neural networks for machine
  learning 4~(2) (2012).

\bibitem{kingma2014adam}
D.~Kingma, J.~Ba, Adam: A method for stochastic optimization, arXiv preprint
  arXiv:1412.6980 (2014).

\bibitem{merity2018regularizing}
S.~Merity, N.~S. Keskar, R.~Socher,
  \href{https://openreview.net/forum?id=SyyGPP0TZ}{Regularizing and optimizing
  {LSTM} language models}, in: International Conference on Learning
  Representations, 2018.
\newline\urlprefix\url{https://openreview.net/forum?id=SyyGPP0TZ}

\bibitem{10.1137/0330046}
B.~T. Polyak, A.~B. Juditsky,
  \href{https://doi.org/10.1137/0330046}{Acceleration of stochastic
  approximation by averaging}, SIAM J. Control Optim. 30~(4) (1992) 838–855.
\newblock \href {https://doi.org/10.1137/0330046} {\path{doi:10.1137/0330046}}.
\newline\urlprefix\url{https://doi.org/10.1137/0330046}

\bibitem{726791}
Y.~Lecun, L.~Bottou, Y.~Bengio, P.~Haffner, Gradient-based learning applied to
  document recognition, Proceedings of the IEEE 86~(11) (1998) 2278--2324.
\newblock \href {https://doi.org/10.1109/5.726791}
  {\path{doi:10.1109/5.726791}}.

\bibitem{le2015simple}
Q.~V. Le, N.~Jaitly, G.~E. Hinton, A simple way to initialize recurrent
  networks of rectified linear units (2015).
\newblock \href {http://arxiv.org/abs/1504.00941} {\path{arXiv:1504.00941}}.

\end{thebibliography}

\clearpage

\appendix
\addcontentsline{toc}{section}{Supplementary Materials}
\renewcommand{\thesubsection}{Appendix \Alph{subsection}}
\section*{Supplementary Materials}
\subsection{Proof of Proposition \ref{prop:gradients}}\label{Supp_Mater:A}

We present below the derivation for the propagation of the gradient through the ATN method.

\renewcommand{\theequation}{A.\arabic{equation}}
\setcounter{prop}{0}
\begin{prop}
Consider ATN for a sequence $\mathbf{a} = \{a^{(t)}\}\subset {\mathbb R}^n$ produced in a RNN and let $y^{(t)}=ATN(\mathbf{a}_k^{(t)};\gamma,\beta)$. Then, for $0 \leq m \leq k-1$, we have:

\begin{equation}
    \resizebox{0.875\hsize}{!}{$
    \dfrac{\partial y_i^{(t)}}{\partial a_i^{(t-m)}} = \gamma\odot\dfrac{\displaystyle\dfrac{\partial a_i^{(t)}}{\partial y_i^{(t-m)}}\dfrac{\partial y_i^{(t-m)}}{\partial a_i^{(t-m)}} - \dfrac{\partial \mu_{t,k}}{\partial a_i^{(t-m)}}}{\sqrt{\sigma_{t,k}^2 + \varepsilon}} - \gamma\odot\dfrac{a_i^{(t)}-\mu_{t,k}}{2\left(\sigma_{t,k}^2+\varepsilon\right)^{3/2}}\dfrac{\partial \sigma_{t,k}^2}{\partial a_i^{(t-m)}}
    $}
\end{equation}
where
\begin{equation}
    \dfrac{\partial \mu_{t,k}}{\partial a_i^{(t-m)}} = \dfrac{1}{nk}\sum_{j=0}^{m}\dfrac{\partial a_i^{(t-j)}}{\partial a_i^{(t-m)}}\\
\end{equation}
and
\begin{equation}
    \resizebox{0.875\hsize}{!}{$
    \dsp\dfrac{\partial \sigma_{t,k}^2}{\partial a_i^{(t-m)}} = \dfrac{2}{nk}\left(\sum_{j=0}^m\left(a_i^{(t-j)}-\mu_{t,k} \right) \dfrac{\partial a_i^{(t-j)}}{\partial a_i^{(t-m)}} - \sum_{j=0}^{k-1}\sum_{s=1}^n\left(a_s^{(t-j)}-\mu_{t,k}\right) \dfrac{\partial \mu_{t,k}}{\partial a_i^{(t-m)}}\right)
    $}
\end{equation}
\end{prop}

\begin{proof}
    Suppose $0 \leq m \leq k-1$ then 
    
    \begin{equation}
        \mu_{t,k}=\dfrac{1}{nk}\sum_{j=0}^{k-1}\sum_{s=1}^{n}a_s^{(t-j)}
    \end{equation}
    
    and
    
    \begin{align}
        \dfrac{\partial \mu_{t,k}}{\partial a_i^{(t-m)}}&=\dfrac{1}{nk}\sum_{j=0}^{k-1}\sum_{s=1}^{n}\dfrac{\partial a_s^{(t-j)}}{\partial a_i^{(t-m)}}\\
        &=\dfrac{1}{nk}\left(\dfrac{\partial a_i^{(t)}}{\partial a_i^{(t-m)}}+\dfrac{\partial a_i^{(t-1)}}{\partial a_i^{(t-m)}}+\cdots+\dfrac{\partial a_i^{(t-m+1)}}{\partial a_i^{(t-m)}}+1\right)\\
        &=\dfrac{1}{nk}\sum_{j=0}^{m}\dfrac{\partial a_i^{(t-j)}}{\partial a_i^{(t-m)}};
    \end{align}
    
    \begin{equation}
        \sigma_{t,k}^2=\dfrac{1}{nk}\sum_{j=0}^{k-1}\sum_{s=1}^{n}\left(a_s^{(t-j)}-\mu_{t,k}\right)^2
    \end{equation}
    
    and
    
    \begin{align}
        \dfrac{\partial \sigma_{t,k}^2}{\partial a_i^{(t-m)}}&=\dfrac{1}{nk}\sum_{j=0}^{k-1}\sum_{s=1}^{n}2\left(a_s^{(t-j)}-\mu_{t,k}\right)\left(\dfrac{\partial a_s^{(t-j)}}{\partial a_i^{(t-m)}}-\dfrac{\partial \mu_{t,k}}{\partial a_i^{(t-m)}}\right)\\
        &=\resizebox{0.7325\hsize}{!}{$\dsp\dfrac{2}{nk}\left(\sum_{j=0}^{m}\left(a_i^{(t-j)}-\mu_{t,k}\right)\dfrac{\partial a_i^{(t-j)}}{\partial a_i^{(t-m)}}-\sum_{{j=0}}^{k-1}\sum_{{s=1}}^{n}\left(a_s^{(t-j)}-\mu_{t,k}\right)\dfrac{\partial \mu_{t,k}}{\partial a_i^{(t-m)}}\right)$};
    \end{align}
    
    \begin{equation}
        y^{(t)}=\gamma\odot\dfrac{a^{(t)}-\mu_{t,k}}{\sqrt{\sigma_{t,k}^2+\varepsilon}}+\beta
    \end{equation}
    
    and
    
    \begin{align}
        \dfrac{\partial y_i^{(t)}}{\partial a_i^{(t-m)}}=\gamma\odot\dfrac{\displaystyle \dfrac{\partial a_i^{(t)}}{\partial a_i^{(t-m)}}- \dfrac{\partial \mu_{t,k}}{\partial a_i^{(t-m)}}}{\sqrt{\sigma_{t,k}^2+\varepsilon}}-\gamma\odot\dfrac{1}{2}\dfrac{a_i^{(t)}-\mu_{t,k}}{\left(\sigma_{t,k}^2+\varepsilon\right)^{3/2}}\dfrac{\partial \sigma_{t,k}^2}{\partial a_i^{(t-m)}}
    \end{align}
    
    where 
    
    \begin{align}
        \dfrac{\partial a_i^{(t)}}{\partial a_i^{(t-m)}} = \dfrac{\partial a_i^{(t)}}{\partial y_i^{(t-m)}}\dfrac{\partial y_i^{(t-m)}}{\partial a_i^{(t-m)}}.
    \end{align}
\end{proof}

\subsection{Invariance properties}\label{Supp_Mater:B}
\setcounter{table}{0}
\renewcommand{\thetable}{B.\arabic{table}}

\begin{table}
    \centering
    \begin{tabular}{l||c|c|c|c|c}
         \toprule
         & BN   & WN    & LN    & ATN-BN    & ATN-LN \\
         \midrule
         Weight matrix re-scaling       & Yes  & Yes   & Yes    & Yes    & Yes \\
         \midrule
         Weight matrix re-centering     & No   & No    & Yes    & No     & No \\
         \midrule
         Weight vector re-scaling       & Yes   & Yes   & No    & Yes    & No \\
         \midrule
         Dataset re-scaling             & Yes   & No    & Yes   & Yes    & Yes \\
         \midrule
         Dataset re-centering           & Yes   & No    & No    & Yes    & No \\
         \midrule
         Single training case re-scaling & No   & No    & Yes   & No     & Yes \\
         \midrule
         Input at a single time re-scaling & No & No    & Yes   & No     & No \\
         \bottomrule
    \end{tabular}
    \caption{Invariance properties under different normalization methods: BN - Batch Normalization~\cite{pmlr-v37-ioffe15}, WN - Weight Normalization~\cite{NIPS2016_ed265bc9}, LN - Layer Normalization~\cite{ba2016layer}, ATN-BN - Assorted-Time Normalization built on BN method, and ATN-LN - Assorted-Time Normalization built on LN method.}
    \label{tab:invariance_properties}
\end{table}{}

In Table \ref{tab:invariance_properties} we provide a summary of invariance properties for several normalization methods. This is an expansion of Table 1 in ~\cite{ba2016layer}. Weight matrix re-scaling and re-centering are the adjustments of the weight matrix by multiplying a constant scaling factor or adding a constant re-scaling factor. Weight vector re-scaling is similar to weight matrix re-scaling, but only adjusts a single vector instead of the entire matrix. Dataset re-centering and re-scaling consist of changing every input example by multiplying or adding a constant. Single training case re-scaling is when the dataset adjustments are applied to just one example. 
Of particular interest is the invariance with respect to the scaling of an input at a single time point, which was referenced in Section \ref{section:theory}. This is one of the invariance property which LN has that its ATN adaptation do not, and we argue that this is one of the reasons that our method improves on LN.

\end{document}